%% file: iclr2022_conference.tex
\documentclass{article} 
\usepackage{iclr2022_conference,times}

\input{math_commands.tex}

\usepackage{amsthm}
\usepackage{wrapfig}

\usepackage{hhline}         
\usepackage{array}
\usepackage{graphicx}
\graphicspath{ {images/} }

\newtheorem{theorem}{Theorem}

\newtheorem{proposition}[theorem]{Proposition}

\newtheorem{remark}{Remark}

\usepackage{hyperref}
\usepackage{url}
\usepackage{textcomp}
\usepackage{amsmath}
\usepackage{amssymb}
\usepackage{graphicx}
\usepackage{subcaption}
\usepackage{natbib}

\graphicspath{ {./images/} }

\title{Efficient privacy-preserving inference \\ for convolutional neural networks}

\author{Han Xuanyuan, Francisco Vargas \& Stephen Cummins \\
Department of Computer Science and Technology \\
University of Cambridge \\
\texttt{\{hx263,fav25,sac92\}@cam.ac.uk} \\
}

\newcommand{\unsim}{\mathord{\sim}}

\iclrfinalcopy 
\begin{document}

\maketitle

\begin{abstract}
The processing of sensitive user data using deep learning models is an area that has gained recent traction. Existing work has leveraged homomorphic encryption (HE) schemes to enable computation on encrypted data. An early work was CryptoNets, which takes 250 seconds for one MNIST inference. The main limitation of such approaches is that of the expensive FFT-like operations required to perform operations on HE-encrypted ciphertext. Others have proposed the use of model pruning and efficient data representations to reduce the number of HE operations required. We focus on improving upon existing work by proposing changes to the representations of intermediate tensors during CNN inference. We construct and evaluate private CNNs on the MNIST and CIFAR-10 datasets, and achieve over a two-fold reduction in the number of operations used for inferences of the CryptoNets architecture.
\end{abstract}

\section{Introduction}

Machine Learning as a Service (MLaaS) is a framework in which cloud services apply machine learning algorithms on user-supplied data to produce an inference result which is then returned to the user. However, the data has to be decrypted before inference, which allows a server-side adversary to have access to the user’s  information. Homomorphic encryption (HE), can be applied to enable inference to be performed on encrypted data, enabling the result to be delivered to the user without risk of the server accessing the original data or the inference result. \textsc{CryptoNets} \citep{cryptonets} was the first application of HE to secure neural network inference, and leveraged the YASHE' scheme to perform MNIST classifications. Their approach requires a notably high number of homomorphic operations (HOPs), with a single MNIST inference requiring $\unsim290000$ homomorphic multiplications and $\unsim250$ seconds of inference latency. Later works utilised ciphertext rotations as opposed to the SIMD packing scheme, enabling convolutional and fully connected layers to be computed using much fewer HOPs \citep{gazelle,delphi}. This has been shown to reduce the inference latency of MNIST models by more than an order of magnitude, bringing confidence that private inference can be practical. \textsc{LoLa} \citep{lola} proposed novel representations for intermediate tensors and their MNIST model requires only 2.2 seconds for one inference.

In this work, we introduce a framework for secure inference on secure CNNs, designed to reduce the number of HOPs required per inference whilst preserving prediction accuracy. We integrate the convolution-packing method from \textsc{LoLa} with the fast matrix-vector product method introduced by Halevi and Shoup \citep{halevi+shoup} and utilised by \cite{gazelle} in their multi-party computation framework. We show that utilising the Halevi-Shoup method allows the use of rotations and ciphertext packing to scale better compared with the representations in \textsc{LoLa}, when applied to larger convolutional layers. We perform a more detailed investigation on the scalability of the methods used in \textsc{LoLa} to larger models and show that they are significantly outperformed by our proposed method. In addition, we compare our framework against \textsc{LoLa} by constructing models for MNIST and CIFAR-10. With the same layer parameters as \textsc{LoLa}, we are able to obtain over a two-fold reduction in the number of HOPs per inference. Our CIFAR-10 model achieves similar accuracy to that of \textsc{LoLa}'s but uses far fewer operations.

\section{Prerequisites} 

\subsection{Homomorphic operations}
 
Several recent HE schemes such as BFV \citep{bfv} and CKKS \citep{ckks} are based on the RLWE problem and support SIMD ciphertext operations. On a high level, such schemes establish a mapping between real vectors and a plaintext space. The plaintext space is usually the polynomial ring $\mathcal{R}={\mathbb{Z}[X]}/(X^N+1)$. In particular, this is a \emph{cyclotomic polynomial ring} $\mathcal{R}={\mathbb{Z}[X]}/(\Phi_M(X))$ where $\Phi_M(X)$ is the $M$-th cyclotomic polynomial and $M=2N$ is a power of two. The \emph{decoding} operation maps an element in $\mathcal{R}$ to a vector that is either real or complex, depending on the scheme used. The \emph{encoding} operation performs the reverse. Plaintext polynomials are encrypted into ciphertext polynomials using a \emph{public key}. The operations of addition and multiplication can be performed over ciphertexts using an \emph{evaluation key}. Since each ciphertext corresponds to a vector of real (or complex) values, a single homomorphic operation between two ciphertexts constitutes an element-wise operation between two vectors. In addition, such schemes support rotations of the slots within a ciphertext, with the use of Galois automorphisms. 

\subsection{Fast encrypted convolution}

The first convolutional layer in a CNN can be represented using convolution-packing (Brutzkus et al., 2019). The convolution of an input image $f$ with a filter $g$ of width $w$, height $h$ and depth $d$ can be formulated as   
\begin{align}
f * g =\sum_{x=0}^{w-1}\sum_{y=0}^{h-1}\sum_{z=0}^{d-1} g[x,y,z] \cdot \mathbf{F}^{(x,y,z)},
\end{align}
where $\mathbf{F}^{(x,y,z)}$ is a matrix such that $\mathbf{F}^{(x,y,z)}_{ij}=f[i+x,j+y,z]$. For an input image $I \in \mathbb{R}^{c_\text{in} \times d_\text{in} \times d_\text{in}}$ feature maps and kernel of window size $k\times k$, the input image is represented as $k^2 \cdot c_\text{in}$ vectors $\textbf{v}_1, \ldots, \textbf{v}_{k^2 \cdot c_\text{in}}$, where $\textbf{v}_i$ contains all elements convolved with the $i$-th value in the filter. Denote corresponding ciphertexts as $\text{ct}_1, \ldots, \text{ct}_{k^2 \cdot c_\text{in}}$. The process of producing the $j$-th output feature map is now reduced to ciphertext-plaintext multiplication of each $\text{ct}_i$ with the $i$-th value in the $j$-th filter. In total, the process requires $k^2 \cdot c_\text{in}$  ciphertext-scalar multiplications per output feature map, leading to a total of $k^2 \cdot c_\text{in} \cdot c_\text{out}$ multiplications. 

\section{Method}

In this section, we present our method for achieving privacy-preserving CNN inference with low numbers of HOPs. In summary, we adopt the fast convolution method from \textsc{LoLa} but compute the intermediate convolutional layers in a network using an efficient matrix-vector product method proposed by \cite{halevi+shoup}. This enables large convolutions to be performed in far fewer ciphertext rotations than \textsc{LoLa}'s approach of computing a rotate-and-sum procedure for each row in the weight matrix. In section 3.1 we explain the approach we use. In section 3.2, we perform an analysis to show the improvements made by our modifications compared with the approach from \textsc{LoLa}. In section 4, we apply our approach to models for the MNIST and CIFAR-10 datasets. 

\subsection{Performing intermediate convolutions}
\label{sec:hsmethod}

Consider the convolution of an image $\mathbf{I}$ with a filter $f$. For simplicity, assume that both the image and filter are square, the vertical and horizontal strides of the filter are equal. Let $\mathbf{I} \in \mathbb{R}^{d_\text{in} \times d_\text{in} \times c_\text{in}}$, and $f \in \mathbb{R}^{k \times k \times c_\text{in}}$.  It is well known that this convolution can be flattened and represented as a matrix-vector product $\mathbf{A} \cdot \mathbf{w} \in \mathbb{R}^{d_\text{out}^2 \cdot c_\text{out}}$ where $\mathbf{A} \in \mathbb{R}^{d_\text{out}^2 \cdot c_\text{out} \times d_\text{in}^2 \cdot c_\text{in}}$ and $\mathbf{w}$ is the flattened input.

\cite{halevi+shoup} introduced an efficient method of computing the encrypted matrix-vector product $\mathbf{A} \cdot \mathbf{v}$ for square $\mathbf{A}$. \textsc{GAZELLE} \citep{gazelle} extended the approach to support rectangular $\mathbf{A} \in \mathbb{R}^{m \times n}$. The method works by decomposing $\mathbf{A}$ into its $m$ diagonals, denoted $\left\{ d_1, d_2, \ldots, d_m \right\}$ such that $d_i=\left[\mathbf{A}_{i, 0}, \mathbf{A}_{i+1, 1}, \ldots, \mathbf{A}_{i + n - 1, n - 1} \right]$. Note that all row positions are in modulo $m$. Then each $d_i$ is \emph{rotated} $i$ positions to align  values belonging to the same row of  $\textbf{A}$ into the same column(s), and finally each rotated diagonal is multiplied with with corresponding rotations of $\textbf{v}$. The ciphertexts are summed, and the last stage is to apply a rotate-and-sum procedure to the resulting ciphertext. Overall, this procedure requires $O(m)$ multiplications and $O(m + \log_2 n)$ rotations. We show that this can be applied to the convolutions in \textsc{LOLA}, but under certain constraints. Neither the original method \citep{halevi+shoup} nor \textsc{GAZELLE} \citep{gazelle} discuss the constraints the ciphertext slot count $N$ imposes in this context. In Proposition \ref{remark:rot} we show that certain constraints must be applied on the sizes of the input and output of the convolution. In particular, we must have $N \ge d_\text{in}^2\cdot c_\text{in} + d_\text{out}^2 \cdot c_\text{out} - 1$, unless $d_\text{out}$ and $c_\text{out}$ are both powers of 2.

\begin{proposition}\label{remark:rot}
Let $N$ denote the ciphertext slot count, and $n=d_\text{in}^2 \cdot c_\text{in}$ be the size of the convolution input, and $m=d_\text{out}^2 \cdot c_\text{out}$ be the size of the output. The basic Halevi-Shoup method, which takes $m-1+\lceil \log_2 {\left(\frac{m+n-1}{m}\right)} \rceil$ rotations, requires the condition that $N \ge m+n-1$. If this does not hold, but it is the case that $m=2^l, \; 0 \le l \le  \log_2 N$, then $m-1+\log_2 \frac{N}{m}$ rotations are required.
\end{proposition}
\begin{proof}
Applying the Halevi-Shoup technique requires $m$ ciphertext diagonals $\mathbf{d}^{(1)}, \ldots, \mathbf{d}^{(m)}$ of length $n$ to be extracted, rotated and summed. Note that $\mathbf{d}_j^{(i)}=\mathbf{A}_{i+j,j}$. Now, if $N \ge m + n - 1$ then $N$ is sufficiently large for all rotations to be performed without wrapping around the ciphertext. If $N < m + n - 1$, then wrap-around will occur for at least one of the diagonal ciphertexts during rotation. For any slot $\mathbf{d}^{(i)}_j$, the rotation by $i$ positions will shift the slot to position $k=i + j\ (\text{mod}\ N)$. The slot $\mathbf{d}^{(i)}_j$ is from the row of $\mathbf{A}$ with index $r=i+j\ (\text{mod}\ m)$, and so must be shifted into an index that is equivalent to $r$ in modulo $m$, in order for the rotate-and-sum algorithm to be used. If $N$ is an integer multiple of $m$, then we see $k \equiv r\ (\text{mod}\ m)$ indeed holds. Otherwise, wrap-around will cause the diagonals to be misaligned when summed together. Since $N$ is a power of 2, the requirement that $m$ divides $N$ is satisfied whenever $d_\text{out}$ and $c_\text{out}$ are also powers of 2 such that $2\log_2 {d_\text{out}}+\log_2{c_\text{out}}\le \log_2 N$.
\end{proof}
Based on Proposition \ref{remark:rot}, we utilise a procedure to ensure that the method can be applied for all choices of $(d_\text{in},d_\text{out},c_\text{in},c_\text{out})$ where $d_\text{out}^2 \cdot c_\text{out} \le N$: if the condition that  $N \ge m+n-1$ does not hold, then we `round' $m$ to the closest power of $2$ not less than itself, and add corresponding rows filled with 0's to the weight matrix. This procedure is illustrated below for the simple case where $N=4,m=3$ and $n=4$. Notice that the padding allows the diagonals to align correctly after rotation. 
\input{figs/squares}

\subsection{Comparison with \textsc{LoLa}}
\label{section:comparison}

\textsc{LoLa} proposes a \emph{dense-vector row-major} matrix-vector product method (which we call \textsc{LoLa}-dense) that maps the input vector in the \emph{dense} representation to an output vector in the \emph{sparse} representation. A vector $\mathbf{v}$ is dense when represented as a single ciphertext where the first $n$ slots correspond to the values in $\mathbf{v}$, and is sparse when represented as $n$ ciphertexts, where the $i$-th ciphertext contains the $i$-th element of $\mathbf{v}$ in all its slots. Let $\mathbf{A}\in \mathbb{R}^{m \times n}$. The dense-vector row-major method computes  $\mathbf{A}\mathbf{v}_1=\mathbf{v}_2$ via a multiplication per row of $\mathbf{A}$ and summing the elements inside each product vector using the rotate-and-sum procedure. 

They also propose a \emph{stacked-vector row-major} method (which we call \textsc{LoLa}-stacked) that requires input vector in the \emph{stacked} representation and provides the output in the \emph{interleaved} representation. A vector $\mathbf{v}$ is stacked when represented as a single ciphertext that contains as many copies of $\mathbf{v}$ as the ciphertext slot count permits, and is interleaved when represented as a single ciphertext similar to the dense representation - but the slots may be shuffled by some permutation. To obtain the stacked representation, they first pack $k$ copies of $\mathbf{v}$ into a single ciphertext where $k={N}/{\delta(n)}$ and $\delta(n)=2^{\lceil \log_2 n \rceil}$ is the smallest power of $2$ greater than or equal to $n$. The stacked vector is then multiplied with corresponding stacked rows of $\mathbf{A}$, before rotate-and-sum is applied. The main drawback of this method is its reliance on the ability to pack many copies of $\mathbf{v}$ into a ciphertext, which requires $N$ to be much larger than $n$. For large convolutions, this is hard to achieve. We derive the number of operations this method takes. 
\begin{remark}
The \emph{stacked-vector row-major} method proposed by \cite{lola} requires $\lceil \frac{m}{k} \rceil  \left( k + \lceil \log_2 n \rceil - 1 \right) + \lceil \frac{m}{k} \rceil - 1$ rotations and $\lceil \frac{m}{k} \rceil$ multiplications.
\end{remark}
\begin{proof}
We have included this in Appendix \ref{appendix:proof}.
\end{proof}

Suppose we are passing a 4096-length representation into a fully connected layer to map to a 64-length embedding, and let $N=16384$. \textsc{LoLa}-dense would require $64 \cdot \log_2 4096 = 768$ rotations and $64$ multiplications; \textsc{LoLa}-stacked would require $1023$ rotations and $64$ multiplications. \input{figs/plot1.tex}Now, the Halevi-Shoup product approach would require only $64+\log_2 8192=77$ rotations and $64$ multiplications. Since rotations are the most expensive operation, we compare the number of rotations required by each of the three methods in Figure \ref{fig:graph}. 

It can be shown that for any $n>1$, \textsc{LoLa}-stacked uses fewer rotations than \textsc{LoLa}-dense. However, \textsc{LoLa}-dense can be used to compute two layers instead of one. For large inputs and output layers (relative to $N$), however, both methods require significantly more rotations than HS. \textsc{LoLa}-stacked relies on $m\cdot\delta(n)/N$ being small whereas \textsc{LoLa}-dense relies on $\log_2 n$ being small. With HS, even if $n=N$, $\log_2 n$ is insignificant compared to $m$. 

In general, we believe that having the number of rotations be linear to only $m$ is beneficial  since neural networks typically down-sample or pool the data to produce denser, higher-level representations, and so can expect $n\ge m$ generally. It should be noted that there are exceptions - such as bottleneck layers.

\section{Experiments and results}

We conduct experiments on the MNIST and CIFAR-10 datasets. We first apply our approach to the \textsc{CryptoNets} architecture used in \textsc{LoLa}, to create a model \textsc{CryptoNets-HS}. The same architecture was shown to achieve close 98.95\% accuracy by \cite{cryptonets}, and we observe similar performance using their training parameters. Training is conducted using TensorFlow \citep{tensorflow}. The architecture is then converted into a sequence of homomorphic operations.  We use the SEAL library \citep{seal} for this. For reference, the original \textsc{CryptoNets} architecture is shown in Appendix \ref{appendix:models}. 
To reduce consumption of instruction depth, linear layers without activations between them are composed together. For instance, each convolution-pooling block of \textsc{CryptoNets-HS} is a single linear layer. This is done to enable our baseline to match the architecture in \textsc{LoLa}. It should be noted that omitting activations can reduce representational power and made not be ideal in practice. In addition, we construct a further-optimised model \textsc{ME} that  has similar test accuracy to \textsc{CryptoNets-HS} but is designed in consideration of the way we are computing the layers. We note that applying the HS method in computing layer 6 of \textsc{CryptoNets-HS} requires setting $m=100$ and $n=845$. Specifically, we reduce the kernel size of the first convolution from $5 \times 5$ to $3 \times 3$, and the size of the first dense layer from $100$ to $32$. To account for the reduced representational power of the first convolution, the stride is reduced from $2$ to $1$. This notably does not add any homomorphic operations to the computation. We are able to achieve 98.7\% test accuracy using \textsc{ME} after 100 epochs of training with the Adam optimiser. 

The models are implemented using operations provided by the SEAL library, and inference is performed on a standard desktop processor. We utilise only a single thread, and measure the number of homomorphic operations required per single inference and model test accuracy as our evaluation metrics.

\newcolumntype{P}[1]{>{\centering\arraybackslash}p{#1}}
\newcommand\tstrut{\rule{0pt}{2.4ex}}


\begin{table}[!htb]
\small 
\begin{tabular}{m{1.9cm} || P{.15cm}P{.15cm}P{.5cm}|P{.15cm}P{.15cm}P{.2cm}|P{.15cm}P{.15cm}P{.4cm}|P{.15cm}P{.15cm}P{.4cm}|P{.15cm}P{.15cm}P{.3cm}|P{.15cm}P{.15cm}P{.2cm}} \hline 
Layer & \multicolumn{3}{c|}{Total HOPs}  & \multicolumn{3}{c|}{Add PC}  & \multicolumn{3}{c|}{Add CC} & \multicolumn{3}{c|}{Mul PC} & \multicolumn{3}{c|}{Mul CC} & \multicolumn{3}{c}{Rot}\\ 
 & $M$  & $L^\prime$ & $L$ & $M$  & $L^\prime$ & $L$ & $M$  & $L^\prime$ & $L$ & $M$  & $L^\prime$ & $L$ & $M$  & $L^\prime$ & $L$ & $M$  & $L^\prime$ & $L$\\
\hline\tstrut
Conv1    & 90 & 250 & 250 & 5 & 5& 5 & 40 & 120 & 120 & 45 & 125 & 125 & - & - & - & - & - & -     \\
Flat1    & 8 & 8 & 8 & - & - & - & 4 & 4 & 4 & - & - & - & - & - & - & 4 & 4 & 4     \\
Square1  & 1 & 1 & 1 & - & - & - & - & - & - & - & - & - & 1 & 1 & 1 & - & - & -     \\
Conv2-Dense1 & 110 & 308 & 492 & 1 & 1 & - & 38 & 103 & 246 & 32 & 100 & 13 & - & - & - & 39 & 104 & 246     \\
Square2 & 1 & 1 & 1 & - & - & - & - & - & - & - & - & - & 1 & 1 & 1 & - & - & -    \\
Dense2 & 36 & 38 & 279 & 1 & 1 & - & 12 & 13 & 139 & 10 & 10 & 10 & - & - & - & 13 & 14 & 130    \\
\hline \tstrut
Total & \textbf{246}&606&1031 & 7&7&\textbf{5} & \textbf{94}&240&509 & \textbf{87}&235&148 & 2&2&2 & \textbf{56}&122&380 \\
\hline
\end{tabular}

\caption{Comparison of operations in three different privacy-preserving MNIST models. $L$ indicates the original \textsc{LoLa}-MNIST model, $L^\prime$ indicates the \textsc{CryptoNets-HS} model, and $M$ indicates the \textsc{ME}  model. `CC' indicates an operation between two ciphertexts and `PC' indicates an operation between a plaintext and a ciphertext.}\label{tab:latency1}



\end{table}

Table 2 shows a breakdown of the operations for the models discussed so far. In homomorphic encryption applied to neural networks, the most expensive homomorphic operation is rotation, with a worst case time complexity of performing both a number theoretic transform (NTT) and an inverse NTT on vectors of length $N$. We observe that \textsc{CryptoNets-HS} requires a total of 122 rotations, as shown in Table \ref{tab:latency1}. \textsc{LoLa}-MNIST requires a total of 380 rotations\footnote{Note that this is deduced to the best of our ability using the descriptions supplied in their paper.} for the same architecture. We notice that at both the convolution-pooling blocks in the \textsc{CryptoNets} architecture, the HS implementation has fewer rotations than \textsc{LoLa}. The reduction in rotations is a direct result of applying the HS method to compute flattened convolutions, in which the number of rotations scales with the size of the output tensor, rather than the size of the input. \textsc{LoLa} uses the stacked-vector row-major method for the first convolution, and the dense-vector row-major method for the second convolution. In Section \ref{section:comparison} we discussed that both of these scale poorly compared to the HS method. In terms of latency, the \textsc{CryptoNets-HS} model requires 2.7 seconds, whereas \textsc{ME} requires only 0.97 seconds. For reference, \textsc{LoLa}-MNIST requires 2.2 seconds per inference; however, they utilise 8 cores on a server CPU whereas our homomorphic operations are run on a single core of a desktop processor.

We also construct a model for the \textsc{CIFAR-10} dataset, called \textsc{CE}. The architecture details are shown in \ref{appendix:tables}. It achieves sightly lower test accuracy ($73.1\%$ vs. $74.1\%$) but requires less than $10\%$ the number of rotations than LoLa's CIFAR-10 model. This is due to both ensuring that intermediate tensor sizes fit well into the number of available slots, and also the use of the HS method. The number of operations for each layer is shown in  \ref{appendix:cifar}. 

A big takeaway from the results is that the application of the HS method to the framework proposed by \textsc{LoLa} can significantly reduce the number of rotations, which are the most computationally expensive HE operation. In addition, the optimised \textsc{ME} is significantly more efficient than \textsc{LoLa} whilst achieving similar accuracy. 

\section{Conclusion}
Privacy-preserving inference using homomorphic encryption is largely constrained by the computational requirements of the operations. We propose improvements over \textsc{LoLa} to achieve lower latencies computing intermediate convolutions, resulting in over a two-fold reduction in the number of rotations for the same MNIST architecture. It is clear that further improvements can made along this direction, especially in the topic of automatically selecting suitable layer parameters to set the trade-off between inference latency and model accuracy \citep{autoprivacy}.
\newpage
\bibliography{iclr2022_conference}
\bibliographystyle{iclr2022_conference}

\newpage

\appendix
\section{Appendix}
\subsection{CryptoNets architecture}
\label{appendix:models}
\input{figs/mnist}

\subsection{Details of our experimental models}
\label{appendix:tables}
\begin{table}[!htb]
\small
\centering

\begin{subtable}{.49\textwidth}
\centering
\begin{tabular}{ P{.6cm} p{1.6cm} p{1.8cm} p{1.6cm}  }
\hline
\textbf{Layer}  & \textbf{Description} & \textbf{Parameters} & \textbf{Input}\\
 \hline
 \hline
  1 & Convolution & $k=3,s=1$ & (1, 30, 30)\\
  - & Square & - & (5, 28, 28) \\
 2 & Avg. Pool & $k=3, s=2$ & (5, 28, 28)\\
 3 & Convolution & $k=3,s=1$ & (50, 14, 14)\\
 4 & Avg. Pool & $k=3, s=2$ & (50, 12, 12) \\
 5 & Flatten & - & (50, 12, 12) \\
 6 & Dense & $m=32$ &  (1250) \\
  - & Square & - & (32) \\
 7 & Dense & $m=10$ &  (32) \\
  - & Softmax & - & (100) \\
 \hline
\end{tabular}
\caption{\textsc{ME} layer parameters.}\label{tab:mnist-opt}
\end{subtable}

\begin{subtable}{1\textwidth}
\centering
\begin{tabular}{ P{1cm} p{2.5cm} p{1.8cm} p{1.6cm} p{1.6cm}  }
\hline
\textbf{Layer}  & \textbf{Description} & \textbf{Parameters} & \textbf{Input} \\
 \hline
 \hline
 1 & Convolution & $k=3,s=1$ & (1, 32, 32) \\
 - & Square & & (18, 30, 30) \\
 2 & Average Pooling & $k=2, s=2$ & (18, 30, 30)  \\
 3 & Sub-convolution & $k=3,s=1$ & (18, 10, 10) \\
 4 & Sub-convolution & $k=1,s=1$ & (13, 8, 8) \\
 - & Square & & (18, 30, 30) \\
 5 & Average Pooling & $k=2, s=2$ & (64, 8, 8) \\
 6 & Convolution & $k=3,s=1$ & (64, 4, 4) \\
 7 & Flatten & - & (256, 2, 2) \\
 8 & Dense & $m=512$ &  (1024)  \\
 - & Square & & (18, 30, 30) \\
 9 & Dense & $m=10$ &  (256) \\
 - & Softmax & - & (10) \\
 \hline
\end{tabular}
\caption{\textsc{CE} layer parameters.}\label{tab:cf-fast}
\end{subtable}

\caption[The original MNIST architecture and our optimised design.]{Layer parameters for \textsc{ME} (left) and \textsc{CE} (bottom). $k$ and $s$ indicate the kernel width and stride respectively. $m$ indicates the number of output nodes of a fully-connected layer.}
\label{table:table_a}
\end{table}

\subsection{Proof of Remark 1}
\label{appendix:proof}
\begin{proof}
The stacking in \cite{lola} is done is using $k-1$ rotations and additions. Then the point-wise multiplication requires a single SIMD multiplication between the ciphertext and $k$ (stacked) rows of $\mathbb{A}$. Finally, $\lceil \log_2 n \rceil$ rotations and additions are performed to compute the $k$ elements of the product vector. Since $N$ is not always sufficiently large to pack $m$ copies of $n$, this procedure has to be performed $\lceil \frac{m}{k} \rceil$ times. To produce all elements of the product vector, $\lceil \frac{m}{k} \rceil  \left( k + \lceil \log_2 n \rceil - 1 \right) $ rotations and $\lceil \frac{m}{k} \rceil$ multiplications are performed, before $\lceil \frac{m}{k} \rceil - 1$ rotations are performed to bring the $\lceil \frac{m}{k} \rceil$ ciphertexts into a single one. Note that the produced ciphertext(s) contain the product elements in an interleaved format, i.e. a permutation. 
\end{proof}

\subsection{Performance on \textsc{CIFAR-10}}
\label{appendix:cifar}

\begin{figure}[h]
\centering
 \vspace*{2mm}
\begin{tabular}{ m{2cm}|| P{.7cm} P{.7cm}  P{.7cm} P{.7cm} P{.7cm} P{.7cm} }
\hline 
Layer  & HOPs & Add PC & Add CC & Mul PC & Mul CC & Rot  \\
 \hline\tstrut 
Conv1 				& 972 & 18& 468  & 486 & - & -    \\
Flat1 				& 30  & - & 15 & -  & - &15  \\
Square1 			& 3  & - & -  & - & 3 & -   \\
Pool1-Conv2   	    & 7506 & 1 & 2504 & 2496 & - &2505  \\
Conv3   			& 1677 & 64 & 768  & 832  & - &13  \\
Flat2    			& 126  & - & 63 & - & - & 63   \\
Square2 			& 1  & - & - & - & 1 & - \\
Pool2-Dense1 		& 778 & 1 & 260 & 256 & - &261  \\
Square  			& 1  & - & - & - & 1 &-   \\
Dense2  			& 40  & 1 & 14 & 10 & - & 15  \\
 \hline \tstrut
 Total & 11134 & 85 & 4092 & 4080 & 5 & 2872  \\
 \hline 
\end{tabular}
\caption{Break-down of operations in \textsc{CE}.}
\label{tab:latency3}
\end{figure}

\end{document}

%% file: math_commands.tex
\usepackage{amsmath,amsfonts,bm}

\def\eqref#1{equation~\ref{#1}}

\def\1{\bm{1}}

\DeclareMathAlphabet{\mathsfit}{\encodingdefault}{\sfdefault}{m}{sl}
\SetMathAlphabet{\mathsfit}{bold}{\encodingdefault}{\sfdefault}{bx}{n}



%% file: figs/squares.tex
\begin{figure}[!htb]
\centering
\begin{subfigure}[t]{\textwidth}
  \centering
  \includegraphics[scale=0.52]{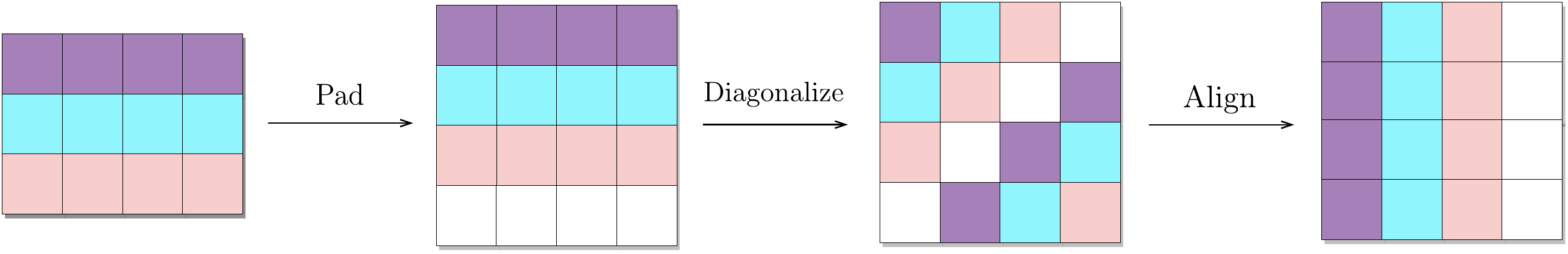}
\end{subfigure}
\caption{An illustration of our version of the Halevi-Shoup method that utilises a padding method.}
\end{figure}

%% file: figs/plot1.tex
\begin{wrapfigure}[23]{r}{0.5\textwidth}
\includegraphics[scale=0.65]{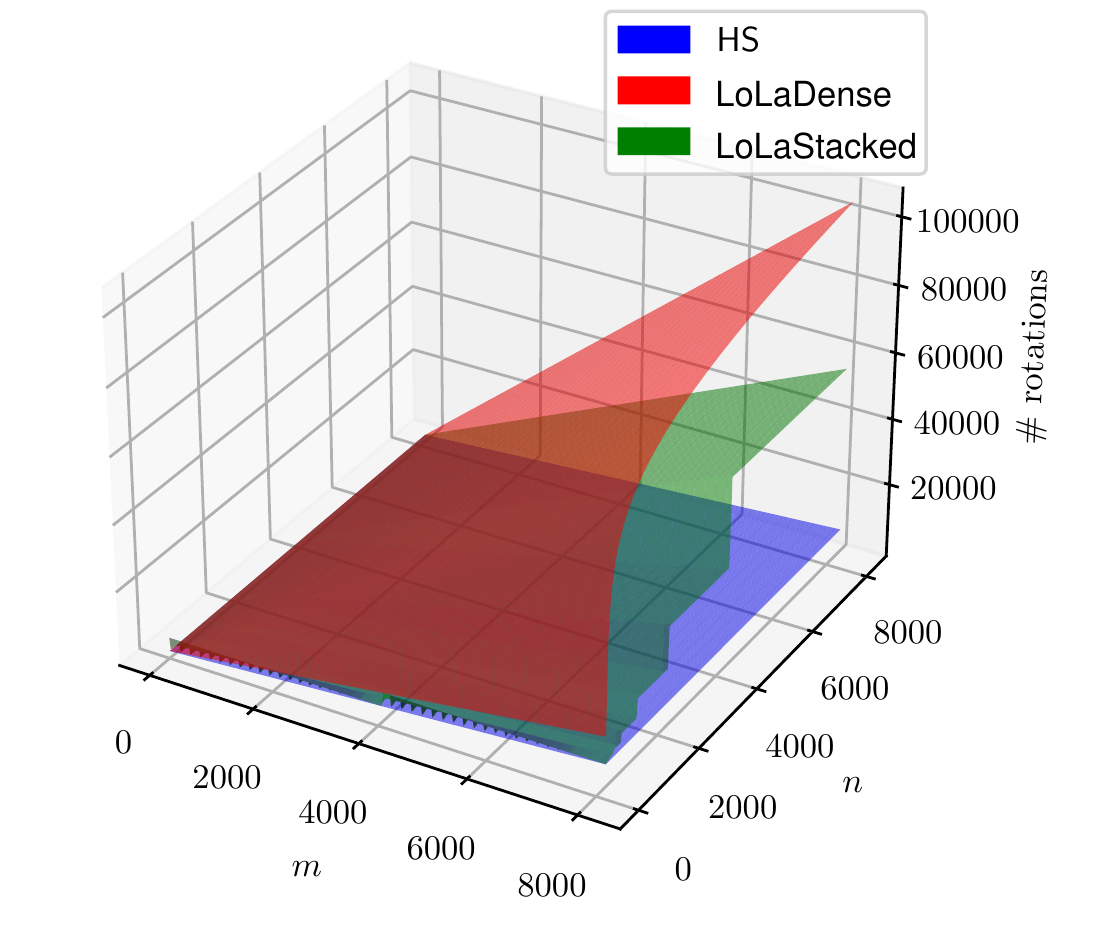}
\caption{Comparison of \textsc{LoLa}'s matrix-vector product methods with the Halevi-Shoup approach, in terms of the number of rotations required for computing a fully-connected layer from $n$ inputs to $m$ outputs.}
\label{fig:graph}
\end{wrapfigure}

%% file: figs/mnist.tex
\begin{figure}[!htb]
\centering
\includegraphics[scale=0.8]{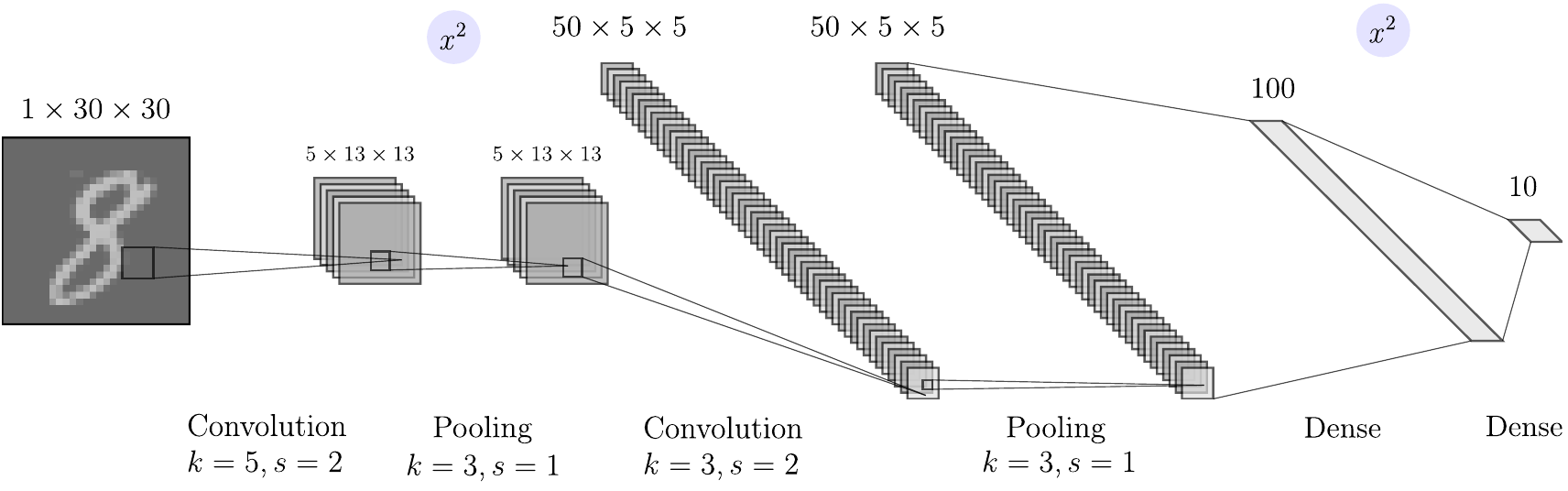}
\caption{The architecture used in LoLa. $k$ indicates the kernel size and $s$ indicates the stride of convolution and pooling. Note that the original $28\times28$ input is padded.} \label{fig:mnist}
\end{figure} 

%% file: iclr2022_conference.bbl
\begin{thebibliography}{10}
\providecommand{\natexlab}[1]{#1}
\providecommand{\url}[1]{\texttt{#1}}
\expandafter\ifx\csname urlstyle\endcsname\relax
  \providecommand{\doi}[1]{doi: #1}\else
  \providecommand{\doi}{doi: \begingroup \urlstyle{rm}\Url}\fi

\bibitem[Abadi et~al.(2015)Abadi, Agarwal, Barham, Brevdo, Chen, Citro,
  Corrado, Davis, Dean, Devin, Ghemawat, Goodfellow, Harp, Irving, Isard, Jia,
  Jozefowicz, Kaiser, Kudlur, Levenberg, Man\'{e}, Monga, Moore, Murray, Olah,
  Schuster, Shlens, Steiner, Sutskever, Talwar, Tucker, Vanhoucke, Vasudevan,
  Vi\'{e}gas, Vinyals, Warden, Wattenberg, Wicke, Yu, and Zheng]{tensorflow}
Mart\'{\i}n Abadi, Ashish Agarwal, Paul Barham, Eugene Brevdo, Zhifeng Chen,
  Craig Citro, Greg~S. Corrado, Andy Davis, Jeffrey Dean, Matthieu Devin,
  Sanjay Ghemawat, Ian Goodfellow, Andrew Harp, Geoffrey Irving, Michael Isard,
  Yangqing Jia, Rafal Jozefowicz, Lukasz Kaiser, Manjunath Kudlur, Josh
  Levenberg, Dandelion Man\'{e}, Rajat Monga, Sherry Moore, Derek Murray, Chris
  Olah, Mike Schuster, Jonathon Shlens, Benoit Steiner, Ilya Sutskever, Kunal
  Talwar, Paul Tucker, Vincent Vanhoucke, Vijay Vasudevan, Fernanda Vi\'{e}gas,
  Oriol Vinyals, Pete Warden, Martin Wattenberg, Martin Wicke, Yuan Yu, and
  Xiaoqiang Zheng.
\newblock {TensorFlow}: Large-scale machine learning on heterogeneous systems,
  2015.
\newblock URL \url{https://www.tensorflow.org/}.
\newblock Software available from tensorflow.org.

\bibitem[Brakerski \& Vaikuntanathan(2011)Brakerski and Vaikuntanathan]{bfv}
Zvika Brakerski and Vinod Vaikuntanathan.
\newblock Efficient fully homomorphic encryption from (standard) lwe.
\newblock In \emph{Proceedings of the 2011 IEEE 52nd Annual Symposium on
  Foundations of Computer Science}, pp.\  97--106. IEEE Computer Society, 2011.

\bibitem[Brutzkus et~al.(2019)Brutzkus, Gilad-Bachrach, and Elisha]{lola}
Alon Brutzkus, Ran Gilad-Bachrach, and Oren Elisha.
\newblock Low latency privacy preserving inference.
\newblock In \emph{Proceedings of the 36th International Conference on Machine
  Learning}, volume~97, pp.\  812--821. PMLR, 2019.

\bibitem[Chen et~al.(2017)Chen, Laine, and Player]{seal}
Hao Chen, Kim Laine, and Rachel Player.
\newblock Simple encrypted arithmetic library - seal v2.1.
\newblock In \emph{Financial Cryptography and Data Security}, pp.\  3--18.
  Springer International Publishing, 2017.
\newblock ISBN 978-3-319-70278-0.

\bibitem[Cheon et~al.(2017)Cheon, Kim, Kim, and Song]{ckks}
Jung Cheon, Andrey Kim, Miran Kim, and Yongsoo Song.
\newblock Homomorphic encryption for arithmetic of approximate numbers.
\newblock In \emph{Advances in Cryptology - ASIACRYPT 2017}, pp.\  409--437.
  Springer, 2017.

\bibitem[Gilad-Bachrach et~al.(2016)Gilad-Bachrach, Dowlin, Laine, Lauter,
  Naehrig, and Wernsing]{cryptonets}
Ran Gilad-Bachrach, Nathan Dowlin, Kim Laine, Kristin Lauter, Michael Naehrig,
  and John Wernsing.
\newblock Cryptonets: Applying neural networks to encrypted data with high
  throughput and accuracy.
\newblock In \emph{Proceedings of The 33rd International Conference on Machine
  Learning}, volume~48, pp.\  201--210. PMLR, 2016.

\bibitem[Halevi \& Shoup(2019)Halevi and Shoup]{halevi+shoup}
Shai Halevi and Victor Shoup.
\newblock Algorithms in helib.
\newblock In \emph{Advances in Cryptology}, pp.\  554--571. Springer, 2019.

\bibitem[Juvekar et~al.(2018)Juvekar, Vaikuntanathan, and
  Chandrakasan]{gazelle}
Chiraag Juvekar, Vinod Vaikuntanathan, and Anantha Chandrakasan.
\newblock Gazelle: a low latency framework for secure neural network inference.
\newblock In \emph{Proceedings of the 27th USENIX Conference on Security
  Symposium}, pp.\  1651--1668. USENIX, 2018.

\bibitem[Lou et~al.(2020)Lou, Bian, and Jiang]{autoprivacy}
Qian Lou, Song Bian, and Lei Jiang.
\newblock Autoprivacy: Automated layer-wise parameter selection for secure
  neural network inference.
\newblock In \emph{Advances in Neural Information Processing Systems},
  volume~33, pp.\  8638--8647. Curran Associates, Inc., 2020.

\bibitem[Mishra et~al.(2020)Mishra, Lehmkuhl, Srinivasan, Zheng, and
  Popa]{delphi}
Pratyush Mishra, Ryan Lehmkuhl, Akshayaram Srinivasan, Wenting Zheng, and
  Raluca~Ada Popa.
\newblock Delphi: A cryptographic inference service for neural networks.
\newblock In \emph{Proceedings of the 2020 Workshop on Privacy-Preserving
  Machine Learning in Practice}, pp.\  27--30. ACM, 2020.

\end{thebibliography}
